\RequirePackage[l2tabu, orthodox]{nag}

\documentclass[utf8]{article}

\usepackage[margin=1.15in]{geometry}

\usepackage{amsmath}
\usepackage{amsthm}
\usepackage{amssymb}
\usepackage{wrapfig}
\usepackage[backgroundcolor=yellow,obeyDraft=true]{todonotes}

\usepackage[math]{blindtext}
\usepackage{mwe}

\usepackage{booktabs}

\usepackage{url}

\usepackage[numbers,sort&compress]{natbib}
\usepackage{microtype}
\usepackage{tikz}
\usetikzlibrary{fit,positioning,arrows,shapes,snakes,automata,%
                petri,%
                topaths,calc}

\newcommand{\primal}[1]{\mathcal{G}_{#1}}
\tikzstyle{tdnode} = [draw,rounded corners,top color=vertexTopColor,bottom color=vertexBottomColor,minimum size=1.5em]
\tikzstyle{stdnode} = [tdnode, font=\scriptsize]
\tikzstyle{stdnodecompact} = [stdnode, inner sep = 1.5pt, outer sep = 0.1pt]
\tikzstyle{stdnodetable} = [stdnode, inner sep = 1.5pt, outer sep = 0]
\tikzstyle{stdnodenum} = [minimum size=1.5em, font=\scriptsize]
\tikzstyle{tdedge} = [-,draw,thick]
\tikzstyle{tdlabel} = [draw=none, rectangle, fill=none, inner sep=0pt, font=\scriptsize]

\tikzstyle{squigarrow} = [->,line join=round,decorate, decoration={
    zigzag,segment length=4,amplitude=.9,post=lineto,post length=2pt}]

\tikzstyle{dashedarrow} = [->,dashed]
\colorlet{vertexTopColor}{white}
\colorlet{vertexBottomColor}{black!10}

\usepackage{xcolor}

\newcommand{\QBFSAT}{\textsc{QSat}\xspace}
\newcommand{\ASP}{\textsc{Asp}\xspace}

\newcommand{\SAT}{\textsc{Sat}\xspace}

\usepackage{array}
\newcolumntype{H}{>{\setbox0=\hbox\bgroup}c<{\egroup}@{}}

\newcommand{\NP}{\ensuremath{\textsc{NP}}\xspace}
\newcommand{\PSPACE}{\ensuremath{\textsc{PSPACE}}\xspace}

\newcommand{\PP}{\ensuremath{\textsc{P}}\xspace}

\newcommand{\FPT}{\ensuremath{\textsc{FPT}}\xspace}

\newcommand{\futuresketch}[1]{}

\newcommand{\normal}{\problemFont{Normal}\xspace}

\newcommand{\tight}{\problemFont{Tight}\xspace}

\newtheorem{theorem}{Theorem}
\newtheorem{proposition}{Proposition}

\newcommand{\Card}[1]{\left|#1\right|}
\newcommand{\CCard}[1]{\|#1\|}

\usepackage{multirow}   
\usepackage{booktabs}   

\newcommand{\cASP}{\ASP}
\newcommand{\sharpSAT}{\textsc{\#}\SAT}

\newcommand{\problemFont}[1]{\textsc{#1}}
\newcommand{\prob}[0]{\ensuremath{\mathtt{P}}}
\newcommand{\inst}[0]{\ensuremath{\mathcal{I}}\xspace}

\usepackage{mathtools}

\DeclareMathOperator{\poly}{\ensuremath{{poly}}}
\newcommand{\Nat}{\ensuremath{\mathbb{N}}}
\newcommand{\bigO}{\ensuremath{{\mathcal O}}}

\newcommand{\twc}[2]{\ensuremath{\mathtt{TW}_{#1}^{#2}}}
\newlength\problemlength
\newcommand\complexityclass[2]{%
\medskip
\begin{center}
\fbox{%
\begin{minipage}{.93	\linewidth}%
\begin{list}{}{\labelwidth\problemlength \labelsep.7em \rightmargin1.5em
\leftmargin\problemlength \advance\leftmargin by3em
\parsep0ex \itemsep.2ex plus.1ex}
\item[{Class:\hfill}] {#1}
\item[{Definition:  \hfill}] #2
\end{list}
\end{minipage}
}
\end{center}
\medskip
}

\def\hy{\hbox{-}\nobreak\hskip0pt}

\usepackage{mathdots}
\usepackage{scalerel}

\DeclareMathOperator{\tower}{\ensuremath{\mathsf{tower}}}

\DeclareMathOperator{\var}{\mathsf{var}}

\title{Advanced Tools and Methods for Treewidth-Based Problem Solving -- Extended Abstract
}
\author{Markus Hecher\footnote{TU Wien, Logic and Computation, Favoritenstraße 9--11,
    1040 Vienna, Austria, \protect\url{hecher@dbai.tuwien.ac.at}. This is an extended abstract of a binational PhD thesis~\protect\cite{Hecher21} that has been awarded with the EurAI Dissertation Award 2021, see~\protect\url{https://www.eurai.org/award/markus-hecher}. 
For a detailed description of the publications that form the basis of the corresponding chapters of the thesis, we refer to the introduction chapter~\protect\cite[Section 1]{Hecher21}.}} 

\begin{document}

\maketitle



\begin{abstract}
Computer programs, so-called solvers, for solving the well-known Boolean satisfiability problem (\SAT) have been improving for decades.
Among the reasons, why these solvers are so fast, is the implicit usage of the formula's structural properties during solving. One of such structural indicators
is the so-called treewidth, which tries to measure how close a formula instance is to being easy (tree-like). 
%
This work focuses on logic-based problems and treewidth-based methods and tools for solving them. Many of these problems are also relevant for knowledge representation and reasoning (KR) as well as artificial intelligence (AI) in general.
We present a new type of problem reduction, which is referred to by \emph{decomposition-guided (DG)}.
This reduction type forms the basis to solve a problem for quantified Boolean formulas (QBFs) of bounded treewidth that has been open since 2004.
The solution of this problem then gives rise to a \emph{new methodology} for proving precise lower bounds for a range of further formalisms in logic, KR, and AI.
Despite the established lower bounds, we implement an algorithm for solving extensions of \SAT efficiently, by directly using treewidth.
Our implementation is based on finding abstractions of instances, which are then incrementally refined in the process.
Thereby, our observations confirm that treewidth is an important measure that should be considered in the design of modern solvers.
%
 %
%
%
%
\end{abstract}

\section{Introduction}

In the last decades, there was a notable progress in solving the well-known \emph{Boolean satisfiability (\SAT)}
problem~\cite{BiereHeuleMaarenWalsh09,KleineBuningLettman99}, which can be witnessed by powerful \SAT solvers 
that are also strikingly fast.
On the one hand, these solvers can decide the existence of a satisfying assignment for Boolean formulas with millions of variables, but on the other hand \SAT is one of the most prominent \NP-complete problems~\cite{Cook71}.
In consequence, this means bad news for solving this problem efficiently, assuming $\PP\neq \NP$, which is the gold standard assumption in computational complexity.
Over the time, even stronger assumptions like the \emph{exponential time hypothesis (ETH)}~\cite{ImpagliazzoPaturiZane01} emerged,
which implies exponential solving time
in the number of variables in the worst case. Nowadays, ETH is widely believed among researchers and therefore oftentimes assumed for establishing theoretical results.
Still, from a scientific point of view, it is not completely clear why in practice \SAT solvers
are dealing so well with a large amount of instances, but there are probably many interleaving reasons
for this observation.
%
%
One of these reasons are \emph{structural properties} of instances that 
are indirectly utilized by the solver's interna, which has been demonstrated at least theoretically~\cite{AtseriasFichteThurley11}.

This thesis deals with such a structural property, which is referred to by treewidth~\cite{RobertsonSeymour86}. 
The \emph{treewidth} is well-studied and measures the closeness of an instance to being a tree (tree-likeness), 
which is motivated by the fact that many
hard problems become easy for the special case of trees or tree-like structures.
This parameter, however, is quite generic and by far not limited to Boolean satisfiability.
In fact, there are further problems parameterized by treewidth that are 
solvable in polynomial time in the instance size when parameterized by treewidth.
Interestingly, also plenty of problems relevant to knowledge representation and reasoning (KR) and artificial intelligence (AI),
which are believed to be even harder than \SAT, can be turned tractable when utilizing treewidth.
One prominent example of such a problem is \QBFSAT, which asks for deciding the validity of a quantified Boolean formula (QBF)~\cite{BiereHeuleMaarenWalsh09,KleineBuningLettman99}, an extension of a Boolean formula where certain variables are existentially or universally quantified.
Complexity-wise it is known that restrictions of this problem reach higher levels of the polynomial hierarchy; in general it is even \PSPACE-complete.
Notably, similar to complexity classes in classical complexity, 
the actual ``hardness'' of such problems when parameterized by treewidth 
is oftentimes quantified by studying precise runtime dependence (levels of exponentiality) on treewidth, see, e.g.,~\cite{PanVardi06,AtseriasOliva14a,MarxMitsou16,LampisMitsou17}. 
\paragraph{Contributions.}
In this work, we study advanced treewidth-based methods and tools for problems in KR and AI.
Thereby, we provide means to establish precise runtime results (\emph{upper bounds}) for prominent 
fragments of the answer set programming (\ASP) formalism, which is a canonical paradigm for solving problems relevant to KR.
Our results are obtained by relying on the concept of dynamic programming 
that is guided along a so-called
tree decomposition in a divide-and-conquer fashion.  
Such a \emph{tree decomposition} is a concrete structural decomposition of an instance, thereby 
adhering to  treewidth. 

Then, we present a new type of problem reduction, which we call a \emph{decomposition-guided (DG) reduction}
that allows us to precisely study and monitor the treewidth increase (or decrease) when
reducing from a certain problem to another problem.
This new reduction type will be the basis for proving a long-open result concerning quantified Boolean formulas. 
Indeed, with this reduction we are able to provide precise \emph{conditional lower bounds} (assuming the ETH) for the problem \QBFSAT when parameterized by treewidth. 
%
More precisely, by relying on DG reductions, we prove that \QBFSAT when restricted to formulas of quantifier rank~$\ell$ 
and treewidth~$k$ cannot be decided in a runtime that is better than~$\ell$-fold exponential in the treewidth
and polynomial in the instance size.\footnote{``$\ell$-fold exponentiality'' refers to a runtime dependence on the treewidth~$k$ that is a tower of~$2$'s of height~$\ell$ with~$k$ on top. More precisely, this indicates runtimes of the form~$\underbrace{2^{{\scaleto{\iddots}{7pt}}^{2^{\bigO(k)}}}}_{\text{height }\ell{+}1}\cdot\poly(n)$, where~$n$ are the number of variables.} 
%
This non-incrementally lifts a known result for quantifier rank~$2$ to arbitrary quantifier ranks,
but yet implies further consequences. 

Even further, the lower bound result for \QBFSAT allows us to design a \emph{new methodology}
for establishing lower bounds for a plethora of problems in the area of KR and AI.
In consequence, we prove that all our upper bounds and DG reductions presented in this thesis
are tight under the ETH.
The lower bound result for \QBFSAT and the resulting methodology also unlocks
a hierarchy of dedicated runtime classes for problems parameterized by treewidth.
These classes can be used to quantify
their hardness for utilizing treewidth and categorize them
according to their runtime dependence on the treewidth.

Finally, despite the devastating news concerning lower bounds,
we are able to provide an efficient implementation of 
algorithms based on dynamic programming that is guided along a tree decomposition.
Our approach works by finding suitable \emph{abstractions} of instances,
which is subsequently refined in a nested (recursive) fashion.
Given the tremendous power of \SAT solvers, our
implementation is \emph{hybrid} in the sense that it heavily
uses such standard solvers for solving certain subproblems
that appear during dynamic programming.
%
%
It turns out that our resulting solver
is quite competitive 
for two canonical counting
problems related to \SAT.
In fact, we are able to solve instances with treewidth upper bounds
beyond 260, which underlines that treewidth might be
indeed an important parameter that should be considered
in modern solver designs.

\medskip

{%
	\begin{table}[t]%
		\fontsize{9}{7.5}\selectfont
		\centering\hspace{-0.0em}%
		\begin{tabular}[t]{@{\hspace{.25em}}p{13em}@{\hspace{-1em}}@{\hspace{.1em}}p{9em}@{\hspace{.2em}}@{\hspace{.2em}}p{7.1em}@{\hspace{.2em}}@{\hspace{.3em}}p{6em}@{\hspace{.25em}}@{\hspace{.15em}}p{6.1em}@{\hspace{-.1em}}}%
\toprule
			\multirow{2}{*}{Problem} &
			\multicolumn{4}{c}{Runtime dependence on treewidth~$k$} 
				\\\cmidrule{2-5}
			& \centering {Exponentiality}  
& \centering {Runtime*} & Upper bound & Lower bound\\\midrule 
			\problemFont{\SAT}
& \centering{single exponential} & \centering{$2^{\Theta(k)}$} & $\vartriangle$\cite{SamerSzeider10b} & $\triangledown$\cite{ImpagliazzoPaturiZane01} \\
			\problemFont{\tight \ASP
} &  \centering{single exponential} & \centering{$2^{\Theta(k)}$} & $\blacktriangle$ Thm.~3.8& $\blacktriangledown$ Prop.~3.9\\
			\problemFont{\normal \ASP
} &  \centering{slightly super exp.} & \centering{$2^{\Theta(k\cdot\log(k))}$} & $\blacktriangle$ \textbf{Thm.~\ref{thm:upperbound}} & $\blacktriangledown$ \textbf{Thm.~\ref{thm:lowerbound}}\\
			\problemFont{$\iota$-Tight \ASP} 
& \centering{slightly super exp.} & \centering{$2^{\Theta(k\cdot\log(\iota))}$} & $\blacktriangle$ Thm.~4.27 & $\blacktriangledown$ Corr.~4.28\\
%
%
			%
			\problemFont{\cASP
}
&  \centering{double exponential} & \centering{$2^{2^{\Theta(k)}}$} & $\vartriangle$\cite{JaklPichlerWoltran09} & $\blacktriangledown$ \textbf{Thm.~\ref{thm:disjasplb}}\\
			\problemFont{$\ell\hy\QBFSAT$}, quantifier rank $\ell$
& \centering{$\ell${-}fold exponential} & \centering{$\tower(\ell, \Theta(k))$} & $\vartriangle$\cite{Chen04a} & $\blacktriangledown$ \textbf{Thm.~\ref{lab:primqbflb}} \\
%
			%
%
\bottomrule
		\end{tabular}\vspace{-.25em}
		\caption{Excerpt of key findings of this work consisting of runtime results (upper bounds) as well as hardness results (lower bounds) for problems parameterized by treewidth. 
%
%
Ideas of bold-faced statements are briefly sketched in this abstract.
The column ``Runtime*'' does not explicitly show factors, which are \emph{polynomial ($\poly(n)$)} in the instance size~$n$.
The \emph{function~$\tower(\ell, k)$} is a tower of exponents of~$2$'s  of height~$\ell$ with~$k$ on top.
Known upper bounds are indicated by~``$\vartriangle$'', whereas new results are marked by~``$\blacktriangle$''.
Lower bounds assume ETH, where new results are indicated by~``$\blacktriangledown$'' and existing results are given by~``$\triangledown$''.
		}
		\label{tab:overview}
	\end{table}%
}

\vspace{-1.5em}
\paragraph{Overview and Structure.}
Table~\ref{tab:overview} provides a short overview on selected key findings of this work as well as some related existing results.
Thereby we show selected lower and upper bounds that are proved in the thesis.
%
The displayed problems of the table are variants of \SAT, \QBFSAT, as well as important fragments of \ASP, which is an essential formalism for modeling problems
in knowledge representation and reasoning.
%
In the course of the thesis, this table is significantly extended, where we also show consequences for many further problems relevant to knowledge representation and artificial intelligence.
We refer to Chapter~6 of the thesis for this extension~\cite[Table 6.1]{Hecher21}.
%

Each section of this abstract briefly summarizes the corresponding chapter of the thesis. 
Note, however, that due to the space limit only a small frame of findings and only key ideas can be sketched in this paper.
In the next section, we recap some basics. 
Then, in Section~\ref{sec:ubs} we establish new upper bound results via dynamic programming for important fragments of logic programs (\ASP).
In Section~\ref{sec:dg} we illustrate the concept of decomposition-guided (DG) reductions, which holds a crucial role in the thesis.
Section~\ref{sec:lbs} presents new lower bounds for \QBFSAT and \ASP, thereby heavily utilizing DG reductions.
The lower bound result for \QBFSAT enables a novel methodology for proving runtime lower bounds that yields to a hierarchy of 
runtime classes for classifying problems depending on their hardness for utilizing treewidth,
which is briefly outlined in Section~\ref{sec:landscape}.
Despite the established runtime bounds, we show in Section~\ref{sec:solving}, how one can still design efficient and competitive solvers
that heavily rely on utilizing treewidth.

\vspace{-.5em}
\section{Preliminaries}\label{sec:prelims}
A \emph{(Boolean) formula}~$F$ in conjunctive normal form is a conjunction of \emph{clauses}, which is a disjunction of variables or the negation thereof, cf.,~\cite{BiereHeuleMaarenWalsh09}. 
The decision problem \SAT concerns about deciding whether for a given formula~$F$ there exists a satisfying assignment.
A \emph{quantified Boolean formula (QBF)}~$Q=\exists V_1. \forall V_2. \cdots \exists V_\ell. F$ is an extension of a Boolean formula
that additionally quantifies variables either existentially or universally.
The \emph{quantifier rank} of a QBF~$Q=\exists V_1. \forall V_2. \cdots \exists V_\ell. F$ is the number~$\ell$ of alternating quantifiers~\cite{KleineBuningLettman99}. We only consider closed formulas, where we have that~$V_1\cup V_2\cup\ldots \cup V_\ell$ coincides with the variables of~$F$. 
The problem~$(\ell-)\QBFSAT$ asks for a given QBF (of quantifier rank~$\ell$) whether it evaluates to true; $\ell{-}\QBFSAT$ is located on the $\ell$-th level of the polynomial hierarchy.

A \emph{(logic) program}~\cite{GebserKaminskiKaufmannSchaub12,
BrewkaEiterTruszczynski11} is a set of rules, where each rule~$r$ is of the form~$H_r\leftarrow B^+_r, B^-_r$ over sets~$H_r$, $B^+_r$ of variables and set $B^-_r$ of (default negated) variables.
A program is \emph{normal} if~$\Card{H_r}=1$ and \emph{tight}, whenever there are no cyclic dependencies over all rules involving variables in~$B^+_r$ and~$H_r$ of a rule~$r$.
Intuitively, the semantics of this formalism, called \emph{answer set programming (\ASP)}, require that whenever for a rule~$r$ every variable in~$B^+_r$ can be derived, but no variable in~$B^-_r$ is derived, then at least one variable in~$H_r$ has to hold. 
In addition, the \ASP formalism imposes a stability criteria on top, which is based on minimizing the set of those derived variables.
As a consequence, for a given program deciding already the existence of such a set of variables, called \emph{answer set}, is believed to be beyond \NP. 
%
%
%

Assume a given formula, QBF, or logic program~$\mathcal{U}$. Then, $\var(\mathcal{U})$ denotes the \emph{set of variables} of~$\mathcal{U}$.
%
%
Further, the \emph{primal graph~$\primal{\mathcal{U}}$} of~$\mathcal{U}$ is an undirected graph, whose vertices are the variables $\var(\mathcal{U})$ with an edge between two variables, whenever those variables appear together in a clause or rule of~$\mathcal{U}$.
%

A \emph{(tree) decomposition}~$\mathcal{T}=(T,\chi)$ of a graph~$\primal{}$ consists of a tree~$T$ and a function~$\chi$, which assigns every node~$t$ in~$T$ a set of nodes in~$\primal{}$~\cite{RobertsonSeymour86}.
Further, $\mathcal{T}$ has to fulfill (i) \emph{Nodes covered: } for every node~$v$ in~$\primal{}$, there is a node~$t$ in~$T$ such that~$v\in\chi(t)$; (ii) \emph{Edges covered: } for every edge~$e$ of~$\primal{}$, there is a node~$t$ in~$T$ with~$e\subseteq\chi(t)$; and (iii) \emph{Connectedness: } for every three nodes~$t_1$, $t_2$, $t_3$ in~$T$, whenever~$t_2$ lies on the unique path between~$t_1$ and~$t_3$, then~$\chi(t_1)\cap\chi(t_3)\subseteq\chi(t_2)$. 
The \emph{width} of~$\mathcal{T}$ is the largest value~$|\chi(t)|$ over all nodes~$t$ in~$T$ and the \emph{treewidth} of~$\primal{}$ is the smallest width among every tree decomposition of~$\primal{}$.

 %


\section{Upper Bounds for Utilizing Treewidth by Dynamic Programming}\label{sec:ubs}

For proving the existence of parameterized algorithms for a problem when considering treewidth, the famous meta-theorem by Courcelle~\cite{Courcelle90} has been established.
While this theorem and several extensions thereof have been oftentimes invoked to prove theoretical results for problems parameterized by treewidth, such theorems do not necessarily provide precise runtime guarantees.
Alternatively, there is a more direct way to exploit treewidth and to obtain concrete upper bounds.
Indeed, one of the most prominent methods to directly utilize treewidth~\cite{RobertsonSeymour86} is by means of \emph{dynamic programming on tree decompositions}~\cite{CyganEtAl15}.
Thereby, the method of dynamic programming~\cite{Bellman54}, which generally refers to breaking down problems in a divide-and-conquer fashion, is guided along a tree decomposition, where the decomposition is traversed in post-order (bottom-up traversal) such that during the traversal a table is computed for each node of the decomposition.
While for a given graph the computation of a tree decomposition of minimal width (treewidth) is \NP-hard, it is possible to efficiently approximate treewidth~\cite{Bodlaender96} and compute a tree decomposition,
and there are also numerous efficient heuristics as well as exact solvers available~\cite{Dell17a}.
The literature distinguishes plenty of research on dynamic programming of tree decompositions for diverse problems and formalisms~\cite{CyganEtAl15}. 
%
%
%
This section concerns such algorithms,
whereby we particularly focus on key fragments of answer set programming, which have not been studied yet.
We briefly sketch the ideas that results in an algorithm for deciding whether a normal program admits an answer set.
It is known that one can encode a normal logic program into a Boolean formula
with a subquadratic overhead in the number of variables by means of so-called level mappings~\cite{LinZhao03,Janhunen06}.
The idea of these level mappings is to encode some derivation order for the variables that are supposed to hold,
thereby avoiding cyclic derivations. 
However, we demonstrate in the thesis that in general the overhead caused by these reductions is unbounded in the treewidth of the primal graph~$\primal{\Pi}$ for a given normal logic program~$\Pi$.
Consequently, it turns out that for designing a dynamic programming algorithm that is guided along a tree decomposition~$\mathcal{T}=(T,\chi)$ of~$\primal{\Pi}$, 
one needs to relax the notion of level mappings.
This idea leads to the concept of \emph{local level mappings}, where we order variables only locally within a node~$t$ of~$T$,
which relaxes the ``global'' order and therefore indicates only some relative order among variables in~$\chi(t)$.
The relaxation to local level mappings 
might result in several computed solutions per answer set of~$\Pi$, thereby potentially losing the one-to-one (bijective) characterization.
However, if the relative orderings corresponding to local level mappings are properly maintained and ``synchronized'' between neighboring nodes of~$t$,
we arrive at an algorithm that is still sufficient for \emph{deciding} whether~$\Pi$ admits an answer set, which yields the following result.

\begin{theorem}[Upper bound for normal programs]\label{thm:upperbound}
Let~$\Pi$ be an arbitrary normal logic program, where the treewidth of~$\primal{\Pi}$ is~$k$.
Then, deciding whether~$\Pi$ admits an answer set can be achieved in time~$2^{{\mathcal{O}(k\cdot\log(k))}}\cdot\poly(\Card{\var(\Pi)})$.
\end{theorem}

Notably, later in Section~\ref{sec:lbs} we show  that significant improvements of this runtime are not expected in general, cf., Theorem~\ref{thm:lowerbound}.
This is indeed different to classical complexity, where \SAT and the decision problem for normal logic programs are both of similar hardness, namely \NP-complete.

\section{Decomposition-Guided Reductions for Treewidth}\label{sec:dg}

There are several problems in logic and artificial intelligence, 
for which dedicated solutions and systems based on utilizing treewidth have been proposed.
This can be witnessed by the existence of specialized implementations, e.g.,~\cite{CharwatWoltran19,FichteHecherZisser19,KiljanPilipczuk18},
but also more general frameworks have been introduced~\cite{BliemEtAl16,BannachBerndt19,LangerEtAl12}.
Interestingly, also one of the winning solvers~\cite{KorhonenJaervisalo21} of the most-recent model counting competition~\cite{FichteHecherHamiti21}, 
which focuses on solving the canonical counting problem \sharpSAT of \SAT, explicitly exploits treewidth and demonstrates surprising performance gains.

%

Inspired and motivated by utilizing treewidth in order to solve problems efficiently, the question is raised, how instances between different formalisms
can be converted (reduced) in a way that preserves structural properties as far as possible.
%
%
Specialized reductions with guarantees for treewidth are not only interesting in theory, but also have practical relevance, as they might enable efficient solving procedures for further formalisms of different areas. As an example, a reduction to \sharpSAT that linearly preserves (or only slightly increases) treewidth would enable other problem formalisms to benefit from utilizing treewidth-based counting solvers. 
%
%
These considerations are addressed by means of specialized reductions, which are referred to by \emph{decomposition-guided (DG)}.
%
%
The idea of these reductions is inspired by dynamic programming on tree decompositions, as briefly introduced in the previous section, where problems are solved in parts by traversing the tree from the leaves towards the root. Analogously, such a DG reduction reduces a given instance in parts, thereby being guided along a tree decomposition  in order to establish guarantees for the treewidth of the resulting instance.
%

The simplified concept of DG reductions is highlighted in Figure~\ref{fig:decompguided2},
where a given instance~$\inst$ of a source problem~$\prob$ and a decomposition~$\mathcal{T}$ of~$\primal{\inst}$ are assumed.
DG reductions have the advantage that in addition to the resulting instance, they automatically give rise to a tree decomposition~$\mathcal{T}'$ of the resulting instance, which immediately establishes the relation between~$\mathcal{T}$ and~$\mathcal{T}'$.
Further, if such a reduction works for any tree decomposition of the source instance, it immediately yields treewidth dependencies and guarantees in the process, thereby oftentimes allowing for simple transformations that preserve the treewidth.
%



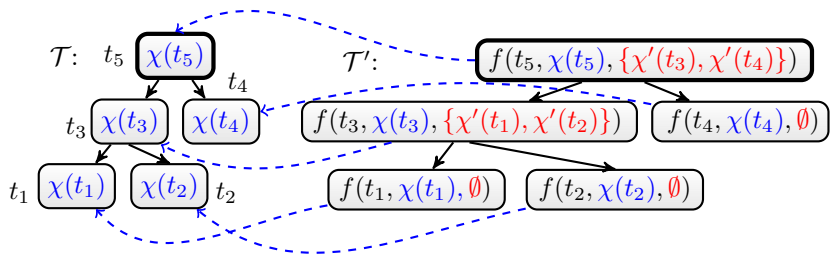
\begin{figure}\vspace{-.1em}
\centering
\begin{tikzpicture}[node distance=1mm, scale=0.15]%
\def\nodedist{0.7em}
\tikzset{every path/.style=thick}
\node (ableft) [tdnode,label={[yshift=-0.25em,xshift=0.25em] left:$t_3$}] {$\textcolor{blue}{\chi(t_3)}$};
\node (leaf1) [below=\nodedist of ableft,xshift=-2em, tdnode,label={[yshift=-0.25em,xshift=0.1em]left:$t_1$}] {$\textcolor{blue}{\chi(t_1)}$};
\node (leaf12) [below=\nodedist of ableft,xshift=1.5em, tdnode,label={[yshift=-0.25em,xshift=-0.1em]right:$t_2$}] {$\textcolor{blue}{\chi(t_2)}$};
\node (leaf2) [tdnode,label={[xshift=-1.0em, yshift=-0.15em]above right:$t_4$}, right = 0.5em of ableft]  {$\textcolor{blue}{\chi(t_4)}$};
\coordinate (middle) at ($ (ableft.north east)!.5!(leaf2.north west) $);
\node (root) [tdnode,ultra thick,label={[]left:$t_5$}, above = \nodedist of middle] {$\textcolor{blue}{\chi(t_5)}$};
\node (llabel) [left=of root,xshift=-1.5em] {$\mathcal{T}$:};
%
\coordinate (top) at ($ (root.north east)+(3.5em,0) $);
\coordinate (bot) at ($ (top)+(0,-4em) $);
%
\draw [stealth'-] (leaf1) to (ableft);
\draw [stealth'-] ($(leaf12.north)+(0.2em,-0.1em)$) to ($(ableft.south)+(-0.25em,0.0em)$);
\draw [-stealth'] (root) to (ableft);
\draw [-stealth'] (root) to (leaf2);
%
%
\node (rleaf1) [right=5em of ableft,tdnode,label={[yshift=-0.25em,xshift=0.25em] left:$ $}] {$\qquad\qquad\quad\qquad\qquad\qquad\quad$};
\node (rleaf1p) [right=5em of ableft,xshift=0.35em,inner sep=0.5] {$f(t_3,\textcolor{blue}{\chi(t_3)}, \textcolor{red}{\{\chi'(t_1), \chi'(t_2)\}})$};
\node (rem1) [below=1em of rleaf1,xshift=-2em, tdnode,label={[yshift=-0.25em,xshift=0.3em]left:$ $}] {$\qquad\qquad\qquad$};
\node (rem1p) [below=\nodedist of rleaf1,yshift=-0.5em,inner sep=0.5,xshift=-2.05em] {$f(t_1,\textcolor{blue}{\chi(t_1),\textcolor{red}{\emptyset}})$};
\node (remab) [below=\nodedist of rleaf1,yshift=-0.3em,xshift=5.5em, tdnode,label={[yshift=-0.25em,xshift=-0.1em]right:$ $}] {$\quad\qquad\qquad\quad$};
\node (remabp) [below=\nodedist of rleaf1,yshift=-0.5em,inner sep=0.5,xshift=5.5em] {$f(t_2,\textcolor{blue}{\chi(t_2)},\textcolor{red}{\emptyset})$};
%
\node (rleaf2) [tdnode,label={[xshift=-0.0em, yshift=-0.15em]above right:$ $}, right = 0.5em of rleaf1]  {$\qquad\qquad\qquad$};
\node (rleaf2p) [right=0.5em of rleaf1,yshift=-0.4em,inner sep=0.5,xshift=.5em,yshift=0.4em,] {$f(t_4,\textcolor{blue}{\chi(t_4)},\textcolor{red}{\emptyset})$};
%
%
\coordinate (middle) at ($ (rleaf1.north east)!.5!(rleaf2.north west) $);
\node (join) [tdnode,ultra thick,label={[xshift=-0.3em]right:$ $}, above  = \nodedist of middle] {\qquad\qquad\qquad\qquad\qquad\qquad\qquad};
\node (llabel) [left=of join,xshift= -3em] {$\mathcal{T}'$:};
\node (joinp) [above = \nodedist of middle,yshift=0.25em,xshift=-.05em,inner sep=0.5] {$f(t_5,\textcolor{blue}{\chi(t_5)}, \textcolor{red}{\{\chi'(t_3), \chi'(t_4)\}})$};
%
\coordinate (top) at ($ (join.north east)+(3.5em,0) $);
\coordinate (bot) at ($ (top)+(0,-4em) $);
%
\draw [stealth'-] (rem1) to (rleaf1);
\draw [stealth'-] ($(remab.north)+(0.2em,-0.1em)$) to ( $(rleaf1.south)+(-0.0em,0.0em)$);
\draw [-stealth'] (join) to (rleaf1);
\draw [-stealth'] ($(join.south)+(0em,0em)$) to (rleaf2);
\draw[dashedarrow,out=-170,in=-50,blue] (rem1) to (leaf1);
\draw[dashedarrow,out=-168,in=-40,blue] (remab) to (leaf12);
\draw[dashedarrow,out=-165,in=-27,blue] (rleaf1) to ($(ableft.south east)+(-2.0em,0.0em)$);
\draw[dashedarrow,blue,out=-191,in=14,blue] (rleaf2) to (leaf2);
\draw[dashedarrow,out=-180,in=30,blue] (join) to ($(root.north)$);
\end{tikzpicture}
\vspace{-1.75em}
\caption{
Simplified illustration of a DG reduction from a source problem~$\prob$ to a destination problem~$\prob'$. Thereby, we assume that an instance~$\inst$ of problem~$\prob$ as well as a tree decomposition~$\mathcal{T}=(T,\chi)$ of~$\primal{\inst}$ is given. The reduction depends on both~$\inst$ as well as decomposition~$\mathcal{T}$ and is therefore \emph{decomposition-guided (DG)}. It is constructed for every node~$t$ of~$T$ and yields a tree decomposition~$\mathcal{T}'=(T,\chi')$ of~$\primal{\inst'}$, where~$\inst'$ is the obtained instance of problem~$\prob'$. Further, we have that for every node~$t$ of~$T$, $\chi'(t)$ functionally depends on~$\chi(t)$ as well as~$\chi'(t')$ for every child node~$t'$ of~$t$ (see~$f$).
%
%
%
%
%
%
%
}\label{fig:decompguided2}
\end{figure}

%

%
%

\section{Lower Bounds by Decomposition-Guided Reductions}\label{sec:lbs}
Recall that in Section~\ref{sec:ubs} we briefly sketched how upper bounds for treewidth-based algorithms are obtained.
Alternatively one can often prove upper bounds for a problem of interest by utilizing a DG reduction to~$\QBFSAT$ that linearly preserves the treewidth, which yields the same upper bounds as those for evaluating QBFs, cf., Table~\ref{tab:overview}.
%
%
Thereby, the natural question of whether one can improve such runtime results arises.
%
 %
As a result, scientists have been establishing lower bounds under the widely believed exponential time hypothesis (ETH) for particular problems, whereby also problems on higher levels of the polynomial hierarchy have been pursued~\cite{MarxMitsou16,LampisMitsou17}.
%
%
%
%
Despite those efforts, a general \emph{methodology} that enables and simply supports the process of proving such conditional lower bounds in a general way.
%
%
We address this shortcoming by providing such a methodology for proving precise conditional lower bounds that is based on the following lower bound result for~$\QBFSAT$, obtained with the help of DG reductions.
%

%
\begin{theorem}[Lower bound for $\ell\hy\QBFSAT$]\label{lab:primqbflb}
Let~$Q$ be an arbitrary QBF of quantifier rank~$\ell$ such that the treewidth
of~$\primal{Q}$ is~$k$. Then, assuming that the ETH holds, one cannot decide whether~$Q$ evaluates to true (or is valid) in time~$\tower(\ell,o(k))\cdot 2^{o(\Card{\var(Q)})}$.
%
\end{theorem}

The proof idea of Theorem~\ref{lab:primqbflb} and the whole approach is innovative, since it uses a DG self-reduction from~$\ell\hy\QBFSAT$ to~$(\ell+1)\hy\QBFSAT$.
In contrast to existing lower bound proof approaches, we exploit an additional quantifier to constructively decrease treewidth exponentially from~$k$ to~$\log(k)$. Techniques that have been used before, oftentimes do not directly relate the parameters, treewidth in our case, of the source and destination instance. The concept is sketched and visualized in Table~\ref{tab:lbs:idea}.

Overall, our approach enables many further conditional lower bound results, which results in a new methodology for proving them. Instead of directly applying ETH by reducing from \SAT, one can simply decide on the quantifier rank~$\ell$ and reduce from~$\ell\hy\QBFSAT$ to the problem of interest. 
If this reduction is carried out via a DG reduction that linearly preserves the treewidth, we immediately establish an $\ell$-fold exponential lower bound. 
Especially for problems higher in the polynomial hierarchy, this greatly simplifies the process, by avoiding a barrier of several levels of exponentiality that one alternatively would have to bypass in a specialized reduction and heavily problem-specific way.
For deeper insights and formal details, we refer to the thesis~\cite[Chapter 5]{Hecher21}.
The thesis also contains a large table of formalisms and corresponding novel conditional lower bounds for treewidth that can be easily obtained by applying this new methodology~\cite[Table 6.1]{Hecher21}.


\begin{table}[t]
		\fontsize{9}{7.5}\selectfont
\centering
\hspace{-.45em}\begin{tabular}[H]{@{\hspace{-1em}}HHH@{\hspace{-.25em}}cc@{\hspace{-1em}}}
\toprule
		& & 
		& \multicolumn{1}{c}{Approach for 2\hy\QBFSAT~\cite{LampisMitsou17}} & \multicolumn{1}{c}{\textbf{Novel approach for $\ell\hy\QBFSAT$}}\\[-.1em]\midrule
		&$(\SAT, {k})$ & $(\SAT, {k})$ & $(\SAT, {k})$ & $(\SAT, {k})$\\[.25em]
		&$\downarrow$ & $\downarrow$ & $\downarrow$ & $\downarrow_{\textbf{DG reduction}}$\\[.5em]
		&$(\prob, {f(n)})$ & $(\prob, {g(k)})$ & $(2\hy\QBFSAT, {\log(n)})$ & $(2\hy\QBFSAT, {\log(k)})$\\\\[-.5em]\cmidrule{4-5}\\[-.5em]
		\multicolumn{2}{c}{}&&{$(\SAT, k)$} & {$(\ell\hy\QBFSAT, {k})$}\\[.25em]
		\multicolumn{2}{c}{}&&{{$\downarrow_{\textbf{?}}$}} & {$\downarrow_{\textbf{DG reduction}}$}\\[.5em]
		\multicolumn{2}{c}{}&&{$(3\hy\QBFSAT, {\log(\log(n)})$} & {$(({\ell{+}1})\hy\QBFSAT, {\log(k))}$}\\\bottomrule
	\end{tabular}\vspace{-.25em}
	\caption{Existing proof approaches for runtime lower bounds are similar to the left column, where ETH is applied by reducing from \SAT.
There, the parameter (treewidth)~$k$ of the formula is not used directly. Instead, the parameter of the destination instance is a function in the number~$n$ of variables of the formula.
The idea of our approach is illustrated in the right column, where treewidth~$k$ of the formula is directly modified by exponentially decreasing it via a DG reduction. The advantage
of our constructive and direct approach is that it can be easily generalized to arbitrary quantifier ranks~$\ell$ (bottom right).
	}\label{tab:lbs:idea}
\end{table}

\subsection*{Are Normal Programs ``Harder'' than \SAT?}\label{sec:asplb}

The results above have immediate consequences also for the \ASP formalism.
Indeed, by utilizing a DG reduction from $2\hy\QBFSAT$ and applying Theorem~\ref{lab:primqbflb}, one can show the following conditional lower bound for logic programs, which closes the gap to the existing upper bound~\cite{JaklPichlerWoltran09}, cf., Table~\ref{tab:overview}. 

\begin{theorem}[Lower bound for logic programs]\label{thm:disjasplb}
Let~$\Pi$ be an arbitrary logic program, where the treewidth of~$\primal{\Pi}$ is~$k$.
Then, assuming the ETH, one cannot decide whether~$\Pi$ admits an answer set in time~$2^{2^{o(k)}}\cdot\poly(\Card{\var(\Pi)})$.
\end{theorem}

Interestingly, with the methodology of the previous section, there are no larger barriers in the course of proving Theorem~\ref{thm:disjasplb}.
%
%
Also, the result is in line with the expectations one might have due to classical complexity theory, as the problem is located on the second level of the polynomial hierarchy%
~\cite{BrewkaEiterTruszczynski11}. 
This is different for the import fragment of \emph{normal logic programs}, where despite the problem of deciding whether a normal program admits an answer set is \NP-complete, its ``hardness'' for treewidth has been open, cf., Table~\ref{tab:overview}.
 %
 %
It has been known that in general a normal program cannot be translated into a Boolean formula such that the answer sets are bijectively captured by the satisfying assignments of the formula, without a subquadratic overhead in the number
of (auxiliary) variables~\cite{LifschitzRazborov06,Janhunen06}. Nevertheless, the following \emph{question has been left open}: Is deciding whether a normal logic program admits an answer set actually ``harder'' than \SAT when considering treewidth?


Indeed, this question can be answered affirmatively when assuming the widely believed ETH. However, the proof is not only more involved than Theorem~\ref{thm:disjasplb}, also its consequences reach further.
%

\begin{theorem}[Lower bound for normal programs]\label{thm:lowerbound}
Let~$\Pi$ be an arbitrary normal logic program, where the treewidth of~$\primal{\Pi}$ is~$k$.
Then, assuming the ETH, one cannot decide whether~$\Pi$ admits an answer set in time~$2^{{o(k\cdot\log(k))}}\cdot\poly(\Card{\var(\Pi)})$.
\end{theorem}

Notably, for treewidth the actual overhead is not only in the number of variables, but this problem is indeed harder to solve under ETH than \SAT.
An informal explanation lies in the observation that there exist programs, whose primal graphs have low treewidth, that can express broader reachability problems
and transitive closures such that the involved variables are required to be widely spread over \emph{any} tree decomposition of low width.


These findings and the algorithm of Section~\ref{sec:ubs} lead to a new family of logic programs, referred to by \emph{$\iota$-tight}.
There, $\iota$ represents for a logic program the ``degree'' between the class of tight ($\iota$=1) and normal ($\iota=k$) programs,
where~$k$ is the treewidth. 
The actual value of~$\iota$ then directly corresponds to the solving effort for treewidth, cf.,~Table~\ref{tab:overview},
which is between the evaluation of tight programs (similar to \SAT) and the evaluation of normal programs.


\section{A Complexity Landscape for Treewidth}\label{sec:landscape}
The findings of the previous section do not only yield a new methodology for proving lower bounds,
it also gives rise to a hierarchy of runtime classes that are useful for categorizing problems according to their immediate
hardness when utilizing treewidth. The basic definition of a family of these runtime classes is given below.

\complexityclass{\twc{i}{k}, for every~$i\in\Nat$}{Class \twc{i}{k} is the set of all problems parameterized by treewidth such that every instance~$\inst$ of these problems can be solved in time~$\tower(i,{\bigO(k)})\cdot\poly(\CCard{\inst})$, where~$k$ refers to the treewidth of ${\primal{\inst}}$ and $\CCard{\inst}$ is the size of~$\inst$.}

The definition of these classes is inspired by the work on general fixed-parameter runtime classes~\cite{DowneyEtAl07}; 
these classes are therefore contained in the broader class \FPT.
Observe that indeed~$\twc{0}{k}=\PP$. It is also immediate by the definition of these classes for any~$i\in\Nat$ we have~$\twc{i}{k}\subseteq \twc{i+1}{k}$.
Even further, under the ETH, the inclusions between those classes~$\twc{i}{k}$ are strict.
Consequently, we show next that these $\twc{i}{k}$ for~$i\in\Nat$ form a strict hierarchy assuming the exponential time hypothesis.

\begin{proposition}
Let~$i\in\Nat$. Then, under the ETH we have that~$\twc{i}{k} \subsetneq \twc{i+1}{k}$. 
\end{proposition}
\begin{proof}
%
Assume towards a contradiction that the ETH holds and at the same time we have that~$\twc{i}{k} \not\subsetneq \twc{i+1}{k}$.
Consequently, due to~$\twc{i}{k} \subseteq \twc{i+1}{k}$, we have that~$\twc{i}{k} = \twc{i+1}{k}$. However, this contradicts Theorem~\ref{lab:primqbflb} for problem~$(i+1)\hy\QBFSAT$,
which is in~$\twc{i+1}{k}$ (cf.,~\cite{Chen04a}), but under the ETH it is not in~$\twc{i}{k}$ since any function in~$\tower(i, \bigO(k))$ is also in the set $\tower(i+1, o(k))$ of functions.
\end{proof}

This hierarchy of runtime classes then yields to a new categorization of problems based on their complexity for treewidth.
These problems range from typical graph problems, extensions of Boolean satisfiability and logic programs, over questions in abstract argumentation and beyond. 
For details, we refer to the table of results~\cite[Table 6.1]{Hecher21}, which also contains results for a list of counting problems.

\section{Efficiently Implementing Treewidth-Aware Algorithms}\label{sec:solving}

Despite the strong lower bound results of Section~\ref{sec:lbs} and consequences of the previous section, we present an approach to design a solver that utilizes treewidth,
but is still capable of dealing with instances of \emph{high treewidth}.
%
%
To this end, we mainly focus on problems based on \SAT and canonical counting extensions like \sharpSAT, which have been gaining increasing importance for quantitative reasoning and AI in general.
%
%
%
%

Note that due to space limits, we can only briefly provide crucial ideas instead of completely discussing full-fledged implementations.
Our approach of efficiently utilizing high treewidth lies on the combination of three key concepts. 
\begin{enumerate}
	\item \textbf{Abstractions: }We are computing certain abstractions of the primal graph representation, which are obtained via heuristics. There, the structured instance parts are subject to being solved by means of dynamic programming that is guided along a tree decomposition of those abstractions. These abstractions are such that they cover only a part of the graph, which can be seen similar to the visual example of zooming out of a larger street map. Then, the remaining graph parts are actually larger sub-instances, where each of these sub-instances is forced to be within a unique tree decomposition node. 
	\item \textbf{Hybrid Solving: }Sub-instances as a result of building such an abstraction that are too unstructured to be tackled by dynamic programming are then solved by existing standard \SAT-based solvers. For those instances, there is not much hope to efficiently utilize structure-based measures like treewidth. However, oftentimes the size of those sub-instances has already been significantly reduced in the process, compared to the full instance. 
	\item \textbf{Nested Refinement: }Sub-instances that still contain some sort of structure are again decomposed, and, if needed, again abstractions are built. So, our approach approximates suitable abstractions of the primal graph that is highly structured (low treewidth). Then, during dynamic programming, sub-instances are simplified and there is again the chance to find some structure. In turn, nested refinement ensures that we can refine abstractions for simpler sub-instances at a later time, namely after simplifications of sub-instances during dynamic programming.
Note that the level of nesting is limited, i.e., if the nesting is too deep, we fall back to hybrid solving. 
\end{enumerate}


Our empirical results of this approach yield the following observations. Notably, our method allows us to solve instances with tree decompositions of widths beyond 260.
From a conceptual point of view, it seems that the hybrid approach is well-suited. Overall, the solving is guided along the basic structure
of the instance such that hardly structured instances are passed to existing \SAT-based solvers.
Further, we indeed observe a difference between problems of single-exponential runtime and those of double-exponential runtime in the treewidth.
For an extension of the problem \sharpSAT that is double-exponential in the treewidth, our approach is capable of successfully utilizing tree decompositions of widths up to 99. 
While this is still remarkable, it practically shows the difference between different levels of exponentiality for treewidth. 
Our observations thereby confirm the study of the classification of Section~\ref{sec:landscape}, where problems are categorized according
to their runtime dependence on the treewidth.

\section{Conclusion}\label{sec:conclusion}
This thesis raises a list of further research questions, which is also reflected by the increasing interest of treewidth%
%
\footnote{``Treewidth'' yields more than 22,000 results on Google Scholar (queried on March 18th, 2022).}.
In the thesis we establish DG reductions, which serve as a key for proving both upper and lower bounds.
%
However, strengths, weaknesses, as well as restrictions and potential extensions of these reductions are widely unexplored. 
Especially in the area of explainable AI, DG reductions could help in proving the correctness of solver runs
for extensions of \SAT when considering treewidth, e.g., counting problems.
%
%
Theorem~\ref{lab:primqbflb} provides a tool for proving precise conditional lower bounds for treewidth,
which we currently generalize to \emph{stronger parameters} as well.
%
%
More recent works lift and extend this result for \emph{constraint programming}, which enables our methodology to express more elaborated
lower bounds~\cite{FichteHecherKieler20}. However, we are certain that these results can be further generalized and applied, e.g., in the context of database theories.
%
%
There are also extensions of the ETH, where we expect that our methodology to yield even more concrete lower bounds.
%
%
In the light of theoretical studies between \SAT solving and potential relations for treewidth, e.g.,~\cite{AtseriasFichteThurley11},
we expect further consequences of the lower bounds of Theorems~\ref{thm:disjasplb} and~\ref{thm:lowerbound}
in the context of utilizing structural properties like treewidth for \ASP solvers.
%
%
Recent efforts of counting-based solvers also manage to combine treewidth and \SAT-based techniques, where treewidth serves as a heuristics within the solver~\cite{KorhonenJaervisalo21}.
We see potential synergies with our approach of Section~\ref{sec:solving}.
%





\bibliography{references}

\end{document}